\documentclass[11pt]{article}
\usepackage[utf8]{inputenc}
\usepackage[T1]{fontenc}
\usepackage{lmodern}
\usepackage{textcomp}
\usepackage[margin=1in]{geometry}
\usepackage{url}
\usepackage{microtype}
\usepackage{xcolor}
\usepackage{graphicx}
\usepackage{booktabs}
\usepackage{amsmath,amssymb,amsfonts,amsthm}
\usepackage{nicefrac}
\usepackage{bm}
\usepackage{array}
\usepackage{tabularx}
\usepackage[round]{natbib}
\usepackage[bookmarks=true,hypertexnames=false]{hyperref}
\usepackage[nameinlink,noabbrev]{cleveref}

\hypersetup{
  colorlinks=true,
  citecolor=blue,
  linkcolor=blue,
  urlcolor=blue
}

\theoremstyle{plain}
\newtheorem{theorem}{Theorem}[section]
\newtheorem{proposition}{Proposition}[section]
\newtheorem{corollary}{Corollary}[section]
\newtheorem{lemma}{Lemma}[section]

\theoremstyle{definition}
\newtheorem{definition}{Definition}[section]

\theoremstyle{remark}

\DeclareMathOperator*{\argmin}{arg\,min}
\DeclareMathOperator*{\argmax}{arg\,max}

\newcommand{\Xnote}[1]{}
\newcommand{\Mnote}[1]{}
\newcommand{\Lnote}[1]{}
\newcommand{\Tnote}[1]{}

\title{Prediction Under Imperfect Compression: \\ A Theory of Approximate MDL}

\author{
Qian Li\\
Shenzhen Research Institute of Big Data\\
\texttt{liqian.ict@gmail.com}
\and
Xinyu Mao\\
University of Southern California\\
\texttt{xinyumao.tcs@gmail.com}
\and
Shang-Hua Teng\\
University of Southern California\\
\texttt{shanghua@usc.edu}
\and
Guangxu Yang\\
University of Southern California\\
\texttt{guangxuy@usc.edu}
}

\begin{document}

\maketitle

\begin{abstract}
Minimum Description Length (MDL) formalizes the principle of Occam's razor by optimizing the total description length: $L(\mathrm{model})+L(\mathrm{data} \ | \ \mathrm{model})$. For sequential prediction, the MDL method repeatedly selects a model with a minimum objective score of the observed prefix for the next step prediction. Classical MDL prediction theory shows that exact optimization of the MDL objective indeed provides a strong compression guarantee that supports reliable prediction. However, practical machine learning usually can only find models by approximately optimizing the objective function. To bridge this gap, this paper addresses the following fundamental question: \emph{Under what forms of approximation and regularization does approximate MDL still guarantee reliable sequential prediction?} This work offers a principled characterization. We prove that for any approximation with \emph{additive} slack \(C\) of the more general form of the \emph{balanced} MDL objective: 
$\lambda\cdot L(\mathrm{model})+L(\mathrm{data} \ | \ \mathrm{model})$, the cumulative expected squared prediction error is finite for all $\lambda\ge1$. The case $\lambda>1$ is proved by an affinity-telescoping argument, while the boundary case $\lambda=1$ is proved by a likelihood-ratio stopping argument based on exact static MDL bounds. Our results establish that classical MDL regularization remains robust to any fixed additive optimization error. Furthermore, we establish that our characterization of the approximate MDL framework is sharp: When $0<\lambda<1$, overfits can happen to incur infinite cumulative expected error in the universal class of estimable measures, and hence a strong form of model-complexity regularization is necessary. In addition, model selection may fail in every regularized regime $\lambda >0$, under multiplicative approximation, and thus, additive approximation is both sufficient and essential.
\end{abstract}

\section{Introduction}
The Minimum Description Length (MDL) principle formalizes Occam's razor by selecting models that give short descriptions of both the model and the observed data in the model \citep{rissanen1978modeling,rissanen1998stochastic,grunwald2007mdl}. MDL has become a central framework connecting model selection, universal coding, and prediction \citep{grunwald2007mdl, vitanyi2000mdl}. In sequential prediction, the principle takes a particularly natural form: after observing a history \(x_{<t}\), the learner selects a model \(\nu\) that 
%trades off 
balances between complexity of the model and its efficiency for explaining the observed prefix, and then uses the conditional distribution \(\nu(\cdot\mid x_{<t})\) to predict the next symbol.

In its natural form, the MDL objective is
$L(\mathrm{model})+L(\mathrm{data} \ | \ \mathrm{model})$, where $L()$ is a length function for measuring the complexity of descriptions.
More general objectives with a regularization parameter of form
$\lambda\cdot L(\mathrm{model})+L(\mathrm{data} \ | \ \mathrm{model})$
are also commonly considered. Under exact optimization, the MDL framework is now well understood. In the realizable setting over countable model classes, \citet{poland2005asymptotics} analyzed discrete MDL for online prediction and established almost-sure convergence together with finite bounds on the quadratic, Hellinger, and Kullback--Leibler losses. \citet{hutter2009discrete} further showed that, for any countable model class containing the true measure, the model selected by discrete MDL eventually makes essentially the same predictions as the true environment in total variation. More recently, \citet{milovanov2024prediction} provided a rigorous proof that a universal predictor based on exact MDL guarantees a finite cumulative prediction error along any single infinite Martin-Löf random sequence when the regularization parameter is large. Taken together, these works formalize the exact-optimization guarantee underlying MDL prediction: Perfect compression based on exact MDL minimization yields reliable sequential prediction.

Despite these theoretical guarantees, exact optimization is often unrealistic in modern machine learning. In practice, finding the best-compressing model corresponds to solving a model-selection or training problem, where computational constraints often force us to settle for approximate solutions. This widening gap between ideal theory and computational reality presents a fundamental challenge. For example, \citet{minsky2014limits} insightfully highlighted both the appeal of algorithmic probability and the need for practical approximations:
\begin{quote}
\textit{It seems to me that the most important discovery since Gödel was the discovery by Chaitin, Solomonoff, and Kolmogorov of the concept called Algorithmic Probability... it is a beautiful theory... but it’s got one problem, that is, that you cannot actually calculate what this theory predicts because it is too hard, it requires an infinite amount of work. However, it should be possible to make practical approximations... that would make better predictions than anything we have today.}
\end{quote}
Minsky's call for practical approximations, combined with the computational intractability of exact MDL, naturally leads to the central question of this paper:
\begin{center}
\emph{Under what forms of approximation and regularization does approximate optimization of the MDL objective still lead to reliable sequential prediction?}
\end{center}

Formally, we consider the realizable sequential prediction framework over a finite alphabet $\mathcal X$ and a countable class $\mathcal M$ of probability measures on $\mathcal {X} ^\infty$. Following exact MDL prediction theory, we equip $\mathcal M$ with positive weights $(w_\nu)_{\nu\in\mathcal M}$ satisfying $\sum_{\nu\in\mathcal M} w_\nu \le 1$, and write $K_w(\nu):=-\log_2 w_\nu$. Given a history $x_{<t}$, we assign model $\nu$ the regularized MDL score
\[
L_{\lambda,w}(\nu;x_{<t})
:=
\lambda K_w(\nu)+\log_2 \nu(x_{<t})^{-1},
\]
where $\lambda > 0$ is the parameter that controls the strength of the %"model complexity" 
regularization. At time $t=1,2,\cdots$, the predictor selects a model $\tilde{\nu}_t$ whose score is approximately minimal, either up to an additive slack $C$ [\emph{additive approximation}] or up to a multiplicative factor $(1+\varepsilon)$ [\emph{multiplicative approximation}], and then predicts $x_t$ with $\tilde{\nu}_t(\cdot\mid x_{<t})$. Our goal is to characterize whether and for what $\lambda$ value the approximate MDL framework still preserves the classical guaranties from-compression-to-prediction.

\subsection{Our contributions}

We answer this fundamental question by identifying a sharp regularization threshold for additive approximation as well as establishing a separation between additively-approximate MDL and multiplicatively-approximate MDL. Formal proofs are given in Section~\ref{sec:additive}, Section~\ref{sec:multi}, and the Appendix.

First, we show that additive approximation is fundamentally reliable, provided $\lambda \geq 1$.

\begin{theorem}[Additive Approx-MDL with Strong Regularization]
\label{thm:intro-positive}
Let \(\mathcal M\) be a countable
class of probability measures on \(\mathcal X^\infty\), equipped with positive
weights \((w_\nu)_{\nu\in\mathcal M}\) satisfying
\(\sum_{\nu\in\mathcal M}w_\nu\le1\). Let \(\mu\in\mathcal M\) be the true
environment. Fix \(C\ge0\). If \(\lambda\ge1\), then every additive-\(C\)
Approx-MDL predictor  \(\rho\) has finite cumulative
expected squared prediction error. More precisely,
\[
\sum_{t\ge1}\mathbb E_\mu\!\left[
\sum_{a\in\mathcal X}
\Bigl(\mu(a\mid X_{<t})-\rho_t(a\mid X_{<t})\Bigr)^2
\right] \le 16K_{\rm PH}\lambda\,2^{C/\lambda}w_\mu^{-1},
\]
Here \(K_{\rm PH}\) is the constant in the exact static MDL quadratic-loss bound of \citep{poland2005asymptotics}; with the normalization used in that bound, one may take \(K_{\rm PH}=21\).
\end{theorem}

We also show that the dependence in Theorem~\ref{thm:intro-positive} is essentially worst-case sharp. 
\begin{proposition}[Near-sharpness of the additive positive bound]\label{prop:additive-positive-sharpness}
There exists an absolute constant \(c>0\) such that the following holds.
For any \(\lambda\ge1\), \(C\ge0\), and \(w\in(0,1)\), there exist a finite
binary model class \(\mathcal M\), weights satisfying
\(\sum_{\nu\in\mathcal M}w_\nu\le1\), a true environment
\(\mu\in\mathcal M\) with \(w_\mu=w\), and a valid additive-\(C\)
Approx-MDL predictor \(\rho\) such that
\[
\sum_{t\ge1}\mathbb E_\mu\!\left[
\sum_{a\in\mathcal X}
\Bigl(\mu(a\mid X_{<t})-\rho_t(a\mid X_{<t})\Bigr)^2
\right]
\ge
c\lambda\Bigl(2^{C/\lambda}(w_\mu^{-1}-1)-1\Bigr)_+,
\]
where \(z_+:=\max\{z,0\}\).
\end{proposition}
The proof of Proposition \ref{prop:additive-positive-sharpness} is given in Appendix~\ref{app:proof-cor-universal}. This general bound in Theorem \ref{thm:intro-positive} immediately yields a guarantee for universal sequence prediction. By specializing to the universal model class and complexity-based priors, we show that the single-step prediction error must vanish over time.

\begin{corollary}[Universal induction under additive Approx-MDL]
\label{cor:intro-universal}
Consider the universal prediction setting where $\mathcal M$ is the class of all estimable probability measures and $w_\nu = 2^{-K(\nu)}$
(where $K(\nu)$ is the Kolmogorov complexity of $\nu$, see
Definition 
\ref{def:measure-complexity}). If the true environment $\mu \in \mathcal{M}$ is estimable and $\lambda \geq 1$, then any valid additive-$C$ Approx-MDL predictor successfully achieves universal induction. That is, its expected single-step squared prediction error converges to zero as $t \to \infty$.
\end{corollary}
The proof of Corollary \ref{cor:intro-universal} is given in Appendix~\ref{app:proof-cor-universal}.
Complementing this positive predictability result, we prove that $\lambda = 1$ is the sharp threshold for
sequential prediction,
%over worst-case instances. 
by proving that when $0< \lambda <1$, simply satisfying the additive Approx-MDL criterion is not sufficient for universal induction because the feasible approximation 
%to the objective 
permits pathological pointwise overfitting to the prefixes data.

\begin{theorem}[Additive Approx-MDL fails for universal induction when \(\lambda\in (0,1)\)]
\label{thm:intro-negative-additive}
Consider the universal prediction setting where \(\mathcal M\) is the class of
all estimable probability measures and \(w_\nu=2^{-K(\nu)}\). For every
\(\lambda\in(0,1)\), there exists an estimable true environment \(\mu\in\mathcal M\), a finite constant \(C<\infty\), and a valid additive-\(C\) Approx-MDL predictor \(\rho\) such that
\[
\sum_{t\ge 1}\mathbb E_\mu\!\left[
\sum_{a\in\mathcal X}
\Bigl(\mu(a\mid X_{<t})-\rho_t(a\mid X_{<t})\Bigr)^2
\right]
=+\infty.
\]
\end{theorem}
While additive approximations are reliable above this threshold, the next result demonstrates a strict separation by showing that multiplicative approximation does not provide a uniform guarantee over all valid
selectors, even in regimes where additive approximation is safe. 

\begin{theorem}[Multiplicative Approx-MDL fails for universal induction]
\label{thm:intro-negative-multiplicative}
Consider the universal prediction setting where $\mathcal M$ is the class of all estimable probability measures and $w_\nu = 2^{-K(\nu)}$. For every $\lambda> 0$ and every $\varepsilon>0$, there exists an estimable true environment $\mu\in\mathcal M$ and a valid multiplicative-$(1+\varepsilon)$ Approx-MDL predictor $\rho$ such that
\[
\sum_{t\ge 1}\mathbb E_\mu\!\left[
\sum_{a\in\mathcal X}
\Bigl(\mu(a\mid X_{<t})-\rho_t(a\mid X_{<t})\Bigr)^2
\right]
=+\infty.
\]
\end{theorem}
The key distinction is that additive and multiplicative approximations create very different effective slack. An additive slack \(C\) remains fixed as the observed prefix grows, and can be absorbed by strengthening the complexity penalty to \(\lambda\geq1\). In contrast, a multiplicative slack permits an additive error of order \(\varepsilon\,\operatorname{OPT}_\lambda(x_{1:n})\), which typically grows linearly with \(n\) in stochastic environments. This growing tolerance can keep systematically wrong models admissible indefinitely.

%Specifically, 
%Table~\ref{tab:approx-mdl-summary} summarizes 
Below, we summarize the prediction guaranties for universal induction (i.e., when $\mathcal M$ is the class of all estimable probability measures and $w_\nu = 2^{-K(\nu)}$): while additive approximation achieves finite error when the %regularization parameter 
$\lambda \geq 1$, multiplicative approximation yields infinite error for any $\lambda > 0$.  
\begin{table}[htbp]
    \centering
    \caption{Summary of prediction guarantees provided by the MDL criterion in the universal estimable class. A positive guarantee means that every selector satisfying the criterion has a finite cumulative expected squared error. A counterexample means that the criterion admits at least one valid selector with infinite cumulative error.}
    \label{tab:approx-mdl-summary}
    \begin{tabular}{@{}lll@{}}
        \toprule
        \textbf{Setting / Approximation} & \textbf{Regularization strength ($\lambda$)} & \textbf{Cumulative Error} \\
        \midrule
        Exact \citep{poland2005asymptotics} & $\lambda = 1$ (Standard) & Finite  \\
        Exact \citep{milovanov2024prediction} & $\lambda > 2$ & Finite  \\
        \midrule
        Additive Approximation (Theorem~\ref{thm:intro-positive}) & $\lambda \geq 1$ & Finite  \\
        Additive Approximation (Theorem~\ref{thm:intro-negative-additive}) & $0<\lambda < 1$ & $+\infty$ \\
        Multiplicative Approximation (Theorem~\ref{thm:intro-negative-multiplicative}) & $\lambda > 0$ & $+\infty$ \\
        \bottomrule
    \end{tabular}
\end{table}

\paragraph{Related work.}
Our work is most directly related to the literature on exact MDL prediction. \citet{poland2005asymptotics} analyze discrete MDL predictors for online prediction, while \citet{hutter2009discrete} proves for any countable class of models, the distributions selected by MDL asymptotically predict the true measure in the class in total variation distance. Our positive theorem (Theorem \ref{thm:intro-positive}) can be viewed as a robustness extension of this exact-selection line to additive approximate minimization. Among existing results, the closest technical precursor is Milovanov's best-explanation predictor \citep{milovanov2024prediction}: there, exact minimization of a complexity-regularized log-loss objective with a strictly large regularization parameter already yields finite cumulative squared error. Our first theorem (Theorem \ref{thm:intro-positive}) shows that this positive phenomenon persists under additive approximation, while our later lower bounds (Theorem \ref{thm:intro-negative-multiplicative}) show that this robustness is specific to additive approximation and does not extend as a uniform guarantee over all valid selectors with multiplicative approximation. 

A second relevant line is algorithmic statistics and best-explanation modeling. Two-part descriptions have long been used to formalize individual explanations through sufficient statistics and structure functions \citep{vitanyi2000mdl,gacs2001algorithmic,vereshchagin2004structure}. This viewpoint is especially relevant for our negative results: when the regularization parameter is small, an approximate MDL selector can explain the observed prefix by encoding the sample itself into the model description, which corresponds to a form of memorization or pointwise overfitting. %(\Tnote{Teng: It is some form and degree of memorization}). 
Closely related in spirit is Hutter's analysis of monotone-complexity prediction \citep{hutter2003monotone}, which showed that universal deterministic or one-part MDL can succeed in deterministic estimable environments yet fail in probabilistic environments. Our additive counterexample for $0<\lambda < 1$ (Theorem \ref{thm:intro-negative-additive}) exhibits a similar overfitting mechanism inside an approximate two-part MDL.

A third line concerns approximation and surrogate MDL objectives. \citet{adriaans2009approximation} showed that along a sequence of shorter approximate two-part MDL codes, the corresponding models need not improve monotonically in goodness of fit. Conceptually, our results may be viewed as a sequential-prediction analogue of this warning: approximate improvement in description length is not, by itself, enough to guarantee reliable prediction. \citet{grunwald2005prequential} study prequential maximum-likelihood codes and related universal coding criteria under misspecification. Their setting differs from ours: we keep the MDL objective itself fixed and vary only the optimization accuracy, measured either additively or multiplicatively at each prediction step.

Finally, MDL-style reasoning also appears in modern machine learning through
description-length regularization, neural-network compression, and prequential
coding
\citep{hinton1993weights,blier2018description,bornschein2023prequential,deletang2024compression}.
These works provide motivation for studying approximate compression-based
selection, but they do not address the selector-level robustness question
analyzed here.

\section{Preliminaries}\label{sec:prelim}

\paragraph{Sequence prediction and the cumulative expected squared prediction error.}
Let \(\mathcal X\) be a finite alphabet. For \(n\ge0\), let \(\mathcal X^n\) denote the set of strings of length \(n\) over \(\mathcal X\), with \(\mathcal X^0\) consisting of the empty string. Let $\mathcal X^*:=\bigcup_{n\ge0}\mathcal X^n$ be the set of all finite strings over \(\mathcal X\). For
\(x\in\mathcal X^*\), let \(\ell(x)\) denote its length, i.e. the unique
integer \(n\ge0\) such that \(x\in\mathcal X^n\). If \(\ell(x)=n\ge1\), we write
\[
x=x_1\dots x_n,
\qquad
x_{1:n}=x,
\qquad
x_{<t}:=x_1\dots x_{t-1}
\quad (1\le t\le n+1).
\]
All logarithms are base \(2\).

A probability measure \(\nu\) on \(\mathcal X^\infty\) is understood as a
probability measure on the product \(\sigma\)-algebra generated by cylinder
sets. For \(x\in\mathcal X^*\), define the cylinder
\[
[x]:=\{\omega\in\mathcal X^\infty:\omega_{1:\ell(x)}=x\},
\]
and abbreviate \(\nu([x])\) by \(\nu(x)\). For the empty string, the corresponding
cylinder is \(\mathcal X^\infty\), and hence its prefix probability is \(1\).
The induced prefix probabilities satisfy the consistency condition
\[
\nu(x)=\sum_{a\in\mathcal X}\nu(xa)
\quad\text{for all }x\in\mathcal X^*.
\]
Whenever \(\nu(x)>0\), define the conditional probability
\[
\nu(a\mid x):=\frac{\nu(xa)}{\nu(x)}.
\]

The sequence prediction problem is the following. An unknown underlying
mechanism \(\mu\) generates an infinite sequence
\[
X_{1:\infty}=X_1X_2\cdots \sim \mu.
\]
Having observed the past data
\[
x_{<n}:=x_1x_2\cdots x_{n-1},
\]
the task is to predict the next symbol \(x_n\). More precisely, the goal is to
estimate the conditional distribution
\[
\mu(\cdot\mid x_{<n}),
\]
or equivalently, the probabilities
\[
\mu(a\mid x_{<n})
:=
\frac{\mu(x_{<n}a)}{\mu(x_{<n})},
\qquad a\in\mathcal X,
\]
whenever \(\mu(x_{<n})>0\).

In this paper, the general positive result only assumes that the true measure
belongs to the countable model class under consideration. In universal induction setting, we consider $\mu$ is an estimable measure, defined
below in Definition~\ref{def:estimable-measure}.

A sequential predictor \(\rho=(\rho_t)_{t\ge1}\) assigns, at each time \(t\) and
after seeing the history \(x_{<t}\), a probability distribution
\[
\rho_t(\cdot\mid x_{<t})
\]
over \(\mathcal X\). Its \emph{predictivity} performance under the true measure \(\mu\) is measured
by the cumulative expected squared prediction error
\[
S_\infty(\mu,\rho)
:=
\sum_{t\ge1}\mathbb E_\mu\!\left[
\sum_{a\in\mathcal X}
\Bigl(\mu(a\mid X_{<t})-\rho_t(a\mid X_{<t})\Bigr)^2
\right],
\]
where all expectations and probabilities are taken with respect to \(\mu\). Let
\[
e_t(\mu,\rho)
:=
\mathbb E_\mu\!\left[
\sum_{a\in\mathcal X}
\Bigl(\mu(a\mid X_{<t})-\rho_t(a\mid X_{<t})\Bigr)^2
\right].
\]
If \(S_\infty(\mu,\rho)<\infty\), then \(e_t(\mu,\rho)\to0\). We say that
\(\rho\) learns \(\mu\) in expected squared one-step error if
\(e_t(\mu,\rho)\to0\), and that \(\rho\) learns a class \(\mathcal C\) if it
learns every \(\mu\in\mathcal C\).

\paragraph{Universal monotone Turing machine and Kolmogorov complexity.}

\begin{definition}[Monotone Turing machine]
A monotone Turing machine is a Turing machine with a one-way read-only input
tape, a one-way write-only output tape, and several two-way work tapes. The input
and output heads never move to the left, and symbols once written on the output
tape are never erased. Thus extending the input can only extend the output.
\end{definition}

To ensure that our complexity measures are objective and machine-independent up
to an additive constant, we fix a universal monotone Turing machine for the
remainder of this paper.

\begin{definition}[Universal monotone Turing machine]
A program for \(U\) is written as
\[
\langle T\rangle p,
\]
where \(T\) is a monotone Turing machine, \(p\) is the binary input to \(T\), and
\(\langle T\rangle\) is a prefix-free encoding of \(T\). On input
\(\langle T\rangle p\), the machine \(U\) simulates \(T\) on input \(p\):
\[
U(\langle T\rangle p) = T(p).
\]
Equivalently, for every monotone Turing machine \(T\), there exists a prefix-free
code \(\langle T\rangle\) such that \(U\) simulates \(T\) with an additive
overhead of \(\ell(\langle T\rangle)\) in program length.
\end{definition}

With the universal machine \(U\) established, we can now formally define the
algorithmic complexity of finite strings.

\begin{definition}[Monotone Kolmogorov complexity]
Fix a universal monotone Turing machine \(U\). For a finite string \(x\), the
monotone Kolmogorov complexity of \(x\) with respect to \(U\) is
\[
K_m(x):=\min\{|p|: U(p)=x*\},
\]
where \(U(p)=x*\) means that the output of \(U(p)\) has \(x\) as a prefix.
\end{definition}

Beyond individual data strings, our analysis involves learning and predicting from underlying distributions. Therefore, we must extend the concept of complexity to quantify the data-generating mechanisms themselves. For the universal induction results, we work over the binary alphabet; any finite alphabet can be encoded into binary strings with only constant-overhead changes in the relevant complexities.

\begin{definition}[Estimable measure]\label{def:estimable-measure}
A probability measure \(\mu\) on \(\{0,1\}^\infty\) is called
\emph{estimable} if there exists a Turing machine which, given
\(x\in\{0,1\}^*\) and a precision parameter \(r\in\mathbb N\), outputs a rational
number \(q\) such that
\[
|q-\mu(x)|\le 2^{-r}.
\]
Equivalently, the cylinder probabilities \(\mu(x)\), \(x\in\{0,1\}^*\), are
uniformly computable to arbitrary precision.
\end{definition}

The next lemma gives the equivalent generative form of estimability.

\begin{lemma}[{\cite[Theorem 4.5.2]{LV97}}]\label{lem:li}
A measure \(\mu\) is estimable if and only if there exists a monotone Turing
machine \(T\) which takes an infinite sequence of i.i.d. fair coin flips as input and outputs a binary sequence distributed according to \(\mu\). 
\end{lemma}

\begin{definition}[Kolmogorov complexity of an estimable measure]
\label{def:measure-complexity}
For an estimable measure \(\mu\), let \(K_U(\mu)\) denote the length of the shortest description \(\langle T\rangle\) of a monotone Turing machine \(T\), on
the reference machine \(U\), that generates \(\mu\) in the sense of Lemma~\ref{lem:li}. When \(U\) is clear from the context, we omit the subscript
and write $K(\mu)$.
\footnote{Equivalently, one may define \(K_U^c(\mu)\) as the length of the shortest program that computes the cylinder probabilities \(\mu(x)\) uniformly
to arbitrary precision. For estimable measures, the generating and computing definitions are equivalent up to an additive constant depending only on the
reference machine: $K_U(\mu)=K_U^c(\mu)\pm O(1)$. Here \(K_U\) is defined via a machine that generates \(\mu\), while \(K_U^c\) is defined via a program that computes \(\mu\). All results in this paper are
unchanged under this \(O(1)\) difference.}
\end{definition}

The following standard coding inequality bounds the monotone complexity of a string using any estimable measure, which will be central to our subsequent
analysis.

\begin{lemma}[{\cite[Theorem~5(i)]{hutter2003monotone}}]
\label{lem:coding-relative-to-measure}
There exists a universal constant \(c\) such that for every estimable measure
\(\nu\) and every \(x \in \{0,1\}^*\),
\begin{equation}\label{eq:coding-relative-to-measure}
    K_m(x) \le K(\nu) - \log_2 \nu(x) + c.
\end{equation}
\end{lemma}

\paragraph{Model priors and MDL.}

Let \(\mathcal{M}\) be an arbitrary countable class of probability measures on
\(\mathcal X^{\infty}\). We equip \(\mathcal{M}\) with positive prior weights
\((w_\nu)_{\nu\in\mathcal M}\) satisfying
\[
\sum_{\nu\in\mathcal M} w_\nu \le 1.
\]

In the universal induction setting, we take \(\mathcal M\) to be the class of
all estimable probability measures. For finite alphabets other than binary, this
is understood after fixing a binary coding of the alphabet. The complexity prior
is defined by
\[
w_\nu := 2^{-K(\nu)}.
\]
Because \(K(\nu)\) is induced by prefix-free descriptions, the weights
\(w_\nu=2^{-K(\nu)}\) satisfy Kraft's inequality,
\[
\sum_{\nu\in\mathcal M}w_\nu\le1,
\]
after choosing one shortest description for each distinct measure. Alternatively,
for a general enumeration \((\nu_i)_{i \ge 1}\) of \(\mathcal{M}\), other typical
choices of priors, see for example \citep{grunwald2007mdl}, include geometric
weights \(w_{\nu_i} = 2^{-i}\) and polynomial weights such as
\(w_{\nu_i} = 6/(\pi^2 i^2)\).

Each prior weight \(w_\nu\) induces a description length, or complexity, given by
\[
K_w(\nu):=-\log_2 w_\nu.
\]
Note that the special choice \(w_\nu=2^{-K(\nu)}\) recovers the usual code-length
formulation, yielding \(K_w(\nu)=K(\nu)\). For a finite observation string \(x\),
the standard two-part description length of \(x\) under model \(\nu\) and prior
\(w\) is defined as
\[
L_w(\nu;x) := K_w(\nu) - \log_2\nu(x),
\]
with the convention that \(-\log_2\nu(x)=\infty\) if \(\nu(x)=0\).

To accommodate varying degrees of regularization, for a parameter \(\lambda>0\),
we define the MDL objective
\[
L_{\lambda,w}(\nu;x)
:=
\lambda K_w(\nu) - \log_2\nu(x)
=
-\log_2\bigl(w_\nu^\lambda\nu(x)\bigr).
\]
When \(\lambda=1\), an exact minimizer of \(L_{1,w}(\cdot;x)\) is precisely the
MDL estimator of \cite[eq.~(10) and Definition~3]{poland2005asymptotics},
namely a maximizer of \(w_\nu\nu(x)\).

For each \(x\in\mathcal X^*\), let
\[
\hat\nu_x^\lambda
\in
\argmin_{\nu\in\mathcal M} L_{\lambda,w}(\nu;x)
=
\argmax_{\nu\in\mathcal M} w_\nu^\lambda\nu(x)
\]
denote an arbitrary exact minimizer, with ties broken arbitrarily. Because
\(\mathcal M\) is countable and \(\sum_{\nu}w_\nu\le1\), for \(\lambda>0\), the
maximum is attained. Indeed, let
\[
s=\sup_{\nu\in\mathcal M} w_\nu^\lambda\nu(x).
\]
If \(s=0\), all models assign zero probability to \(x\); in this degenerate case we define \(\hat\nu_x^\lambda\) arbitrarily. This case is irrelevant in the realizable positive theorem, since on \(\mu\)-typical histories \(s>0\), and it is also irrelevant in the universal lower bounds because deterministic extensions give positive mass to every finite prefix. If \(s>0\), then any model with
\(w_\nu^\lambda\nu(x)\ge s/2\) must satisfy
\[
w_\nu\ge (s/2)^{1/\lambda}.
\]
Since \(\sum_\nu w_\nu\le1\), there are only finitely many such models, and the
supremum over this finite set is a maximum. Hence, for every fixed \(x\) and
\(\lambda>0\), the supremum of \(w_\nu^\lambda\nu(x)\) is attained,
guaranteeing that exact minimizers exist.

Finally, the corresponding MDL predictor is given by
\[
\rho_t^\lambda(a\mid x_{<t})
:=
\hat\nu_{x_{<t}}^\lambda(a\mid x_{<t}).
\]
On histories outside the support of the selected model, the conditional
distribution may be defined arbitrarily; under the realizability assumption and
on \(\mu\)-typical histories, this ambiguity does not affect the results.

The following elementary technical lemma bounds the squared loss in terms of an
affinity gap and plays a key role in the proof of our positive results. It
follows from the Pinsker-type inequality for Csisz\'ar \(f\)-divergences, see
for example \citep[Corollary~9]{gilardoni2010pinsker}. For completeness, a full proof is provided in the appendix.

\begin{lemma}[Affinity gap controls squared loss]\label{lem:affinity-to-l2}
For any \(\alpha\in(0,1)\) and any distributions \(p,q\) on \(\mathcal X\),
\[
\sum_{a\in\mathcal X}(p_a-q_a)^2
\le
\frac{2}{\alpha(1-\alpha)}
\left(1-\sum_{a\in\mathcal X} p_a^{1-\alpha}q_a^\alpha\right).
\]
\end{lemma}

%\section{Additive Approx-MDL}\label{sec:additive}

\section{Reliable prediction: MDL under additive approximation}\label{sec:additive}

This section studies additive approximation to the MDL objective \(L_{\lambda,w}\). The positive result is stated for an arbitrary countable class \(\mathcal M\) of probability measures on \(\mathcal X^\infty\). For the
lower bound, we specialize to the universal prediction setting in which \(\mathcal M\) is the class of all estimable probability measures and \(w_\nu=2^{-K(\nu)}\).

\begin{definition}[Additive Approx-MDL]\label{def:amdl}
Fix \(\lambda> 0\) and \(C\ge0\). Given a history \(x_{<t}\), an additive
Approx-MDL selector chooses a model \(\nu_t\in\mathcal M\) such that
\begin{equation}\label{eq:additive-approx-def}
L_{\lambda,w}(\nu_t;x_{<t})
\le
\min_{\nu\in\mathcal M}L_{\lambda,w}(\nu;x_{<t})+C.
\end{equation}
Equivalently,
\begin{equation}\label{eq:additive-approx-maxform}
w_{\nu_t}^\lambda\nu_t(x_{<t})
\ge
2^{-C}\max_{\nu\in\mathcal M}\{w_\nu^\lambda\nu(x_{<t})\}.
\end{equation}
The corresponding predictor is
\[
\rho_t(\cdot\mid x_{<t}):=\nu_t(\cdot\mid x_{<t}).
\]
If \(\nu_t(x_{<t})=0\), we define \(\rho_t(\cdot\mid x_{<t})\) arbitrarily; on
\(\mu\)-typical histories in the realizable setting of our positive theorem, the
approximation condition implies \(\nu_t(x_{<t})>0\).
\end{definition}

\subsection{When \texorpdfstring{\(\lambda\ge1\)}{lambda >= 1}: finite cumulative error}
\paragraph{Proof idea of Theorem~\ref{thm:intro-positive}.}
The additive-\(C\) condition turns model selection into likelihood-ratio threshold events. Indeed, whenever \(\nu\) is selected,
\[
\frac{\nu(X_{<t})}{\mu(X_{<t})}
\ge
2^{-C}\left(\frac{w_\mu}{w_\nu}\right)^\lambda .
\]
For each fixed pair \((\mu,\nu)\), Appendix~\ref{app:additive-proofs} gives two ways to control the loss on such threshold events: an affinity-telescoping bound of order \(c^{-\alpha}\) for \(0<\alpha<1\), and a critical likelihood-ratio
stopping bound of order \(c^{-1}\) based on the exact static MDL result of \citet{poland2005asymptotics}. Taking the better of the two pairwise bounds and interpolating over the weight ratio \(w_\nu/w_\mu\) gives
\[
S_\infty(\mu,\rho)
\le 16K_{\rm PH}\lambda\,2^{C/\lambda}w_\mu^{-1},
\]
with the explicit constant stated in Theorem~\ref{thm:intro-positive}. The full proof is in Appendix~\ref{app:additive-proofs}.

\subsection{When \texorpdfstring{\(0< \lambda<1\)}{0 < lambda < 1}: pointwise overfitting and infinite cumulative error}

The positive result above shows that the classical boundary \(\lambda=1\) is still safe under additive approximation. The lower bound begins only below this boundary. Throughout this subsection we work in the universal
prediction setting: \(\mathcal M\) is the class of all estimable probability
measures on \(\{0,1\}^\infty\), and \(w_\nu=2^{-K(\nu)}\). We write
\[
L_\lambda(\nu;x):=\lambda K(\nu)+\log_2\nu(x)^{-1}
\]
and
\[
\operatorname{OPT}_\lambda(x)
:=
\min_{\nu\in\mathcal M}L_\lambda(\nu;x).
\]

\begin{lemma}[Pointwise overfitting under weak regularization]\label{lem:pointwise-overfitting}
Fix \(\lambda\in(0,1)\). There exists a constant \(C_\lambda<\infty\),
depending only on \(\lambda\) and the reference machines, such that for every
finite string \(x\in\{0,1\}^*\), there exists a deterministic estimable measure
\(\nu_x\in\mathcal M\) satisfying \(\nu_x(x)=1\) and
\[
L_\lambda(\nu_x;x)
\le
\operatorname{OPT}_\lambda(x)+C_\lambda.
\]
Moreover, there exists a constant \(c_\lambda<\infty\) such that for all
\(x\in\{0,1\}^*\),
\[
\operatorname{OPT}_\lambda(x)
\ge
\lambda K_m(x)-c_\lambda.
\]
\end{lemma}

\paragraph{Proof idea of Theorem~\ref{thm:intro-negative-additive}.}
Lemma~\ref{lem:pointwise-overfitting} shows that when \(0< \lambda<1\), a
deterministic measure that extends the observed prefix can remain within a
constant additive slack of the optimum. Taking the true environment to be the
fair coin, an Approx-MDL selector can choose such deterministic overfitting
measures at every history. Its next-step prediction is then deterministic at
every history, while the true conditional distribution remains \((1/2,1/2)\),
giving constant squared loss. Full details are in Appendix~\ref{app:additive-proofs}.

\section{MDL under multiplicative approximation is not a universal predictor}\label{sec:multi}

In this section, we transition from additive approximations to multiplicative approximations. We analyze this within the framework of universal sequence prediction, where $\mathcal M$ is the class of all estimable probability measures and the weights are given by the canonical choice $w_\nu=2^{-K(\nu)}$. The objective thus becomes
\[
L_{\lambda}(\nu;x) := \lambda K(\nu) + \log_2\nu(x)^{-1}.
\]

\begin{definition}[Multiplicative Approx-MDL]\label{def:multi-amdl}
Fix $\lambda\geq 0$ and $\varepsilon\ge0$. Given a history $x_{<t}$, a multiplicative
Approx-MDL selector chooses an estimable model $\nu_t\in\mathcal M$ such that
\begin{equation}\label{eq:multiplicative-approx-def}
L_{\lambda}(\nu_t;x_{<t})
\le
(1+\varepsilon)
\min_{\nu\in\mathcal M}L_{\lambda}(\nu;x_{<t}).
\end{equation}
The corresponding predictor is
\[
\rho_t(\cdot\mid x_{<t}) := \nu_t(\cdot\mid x_{<t}).
\]
\end{definition}

\begin{lemma}[Universal MDL rarely beats the truth by much]\label{lem:weighted-mdl-lower-tail}
Let $\lambda\ge1$ and define
\[
\operatorname{OPT}_{\lambda}(x)
:=
\min_{\nu\in\mathcal M}L_{\lambda}(\nu;x).
\]
If data is generated by any estimable true environment $\mu \in \mathcal{M}$, then for every $n\in\mathbb N$ and every $c>0$,
\[
\Pr_{X_{1:n}\sim\mu}\!\left[
\operatorname{OPT}_{\lambda}(X_{1:n})
<
\log_2\mu(X_{1:n})^{-1}-c
\right]
\le 2^{-c}.
\]
Equivalently, with probability at least $1-2^{-c}$,
\[
\operatorname{OPT}_{\lambda}(X_{1:n})
\ge
\log_2\mu(X_{1:n})^{-1}-c.
\]
\end{lemma}

\paragraph{Proof idea of Theorem~\ref{thm:intro-negative-multiplicative}.}
The proof separates two regimes. When $0<\lambda<1$, the additive overfitting construction becomes multiplicatively admissible because the optimal score diverges in probability along fair-coin sequences, so a constant additive slack is eventually absorbed by the multiplicative tolerance. When $\lambda\ge1$, a slightly biased rational Bernoulli model remains within the multiplicative tolerance with constant probability because the allowed additive slack grows linearly with the optimal score. In each case, the expected one-step squared error is bounded away from zero infinitely many times. The full proof is given in Appendix~\ref{app:multiplicative-proofs}.

\section{Conclusion}
\label{sec:limitations}

We studied whether approximate minimization of the MDL objective preserves the classical compression--prediction connection in realizable sequential prediction. The results show that the answer depends sharply on both the form of approximation and the regularization strength. Fixed additive approximation is prediction-safe at and above the classical boundary \(\lambda=1\): every valid additive-\(C\) selector has finite cumulative expected squared error. This boundary is sharp in a worst-case selector sense, since for every \(0<\lambda<1\) the additive Approx-MDL criterion admits pathological selectors that overfit individual prefixes. In contrast, multiplicative approximation does not provide a uniform prediction guarantee over all valid selectors for any \(\lambda>0\). Thus approximate compression remains predictive only when the optimization error is controlled in the right scale.

%\paragraph{Limitations.}
%Our results are theoretical and focus on realizable sequential prediction. The positive result applies to countable model classes containing the true environment and gives an expected squared one-step error guarantee; it does not address misspecification, uncountable model classes except through countable encodings or discretizations, high-probability finite-time bounds, or other losses. The paper also does not provide an efficient algorithm for finding approximate MDL selectors in large-scale model classes. The lower bounds are existential: they show that the Approx-MDL criterion alone permits bad valid
%selectors, not that every practical approximate optimizer fails in the corresponding regimes. Additional assumptions on tie-breaking, selector stability, computability, or optimization dynamics may change the behavior and are left for future work.

%\paragraph{Broader impacts.}
%This work is foundational and does not release data, models, or deployed systems. Its main potential positive impact is to clarify when compression-based model selection can be trusted as a prediction principle. The main indirect risk is overinterpretation: using compression or approximate MDL scores as evidence of predictive reliability outside the assumptions analyzed here could lead to misleading evaluations, especially under weak regularization or multiplicative approximation.

\bibliographystyle{plainnat}
\bibliography{reference}

\appendix
\section{Full proofs for additive Approx-MDL}
\label{app:additive-proofs}

\begin{proof}[Proof of Lemma~\ref{lem:affinity-to-l2}]
Fix $\alpha\in(0,1)$. We first prove the following scalar inequality: for every $x,y\in[0,1]$,
\begin{equation}
\label{eq:scalar-affinity-l2}
(1-\alpha)x+\alpha y-x^{1-\alpha}y^\alpha
\ge
\frac{\alpha(1-\alpha)}{2}(x-y)^2.
\end{equation}

If $x=y=0$, the claim is immediate. Otherwise set
\[
M:=\max\{x,y\}\in(0,1],
\qquad x=Ma,\qquad y=Mb.
\]
Then $a,b\in[0,1]$ and $\max\{a,b\}=1$. Moreover,
\[
(1-\alpha)x+\alpha y-x^{1-\alpha}y^\alpha
=
M\Bigl((1-\alpha)a+\alpha b-a^{1-\alpha}b^\alpha\Bigr),
\]
and
\[
\frac{\alpha(1-\alpha)}{2}(x-y)^2 =
\frac{\alpha(1-\alpha)}{2}M^2(a-b)^2
\].

Since $M\le 1$, we have $M^2\le M$. Hence, it is enough to prove
\begin{equation}
\label{eq:normalized-affinity-l2}
(1-\alpha)a+\alpha b-a^{1-\alpha}b^\alpha
\ge
\frac{\alpha(1-\alpha)}{2}(a-b)^2
\end{equation}
under the normalization $\max\{a,b\}=1$.

By symmetry, it suffices to consider the case $b=1$. Indeed, the case $a=1$ follows from the case $b=1$ after swapping $a$ and $b$ and replacing
$\alpha$ by $1-\alpha$. Thus let $b=1$ and write $a=s\in[0,1]$.
Define
\[
g(s):=(1-\alpha)s+\alpha-s^{1-\alpha}
-\frac{\alpha(1-\alpha)}{2}(1-s)^2.
\]
We claim that $g(s)\ge0$ for all $s\in[0,1]$. Clearly, $g(1)=0$.

For $s\in(0,1]$, a direct computation gives
\[
g'(s) =(1-\alpha)-(1-\alpha)s^{-\alpha}
+\alpha(1-\alpha)(1-s),
\]
and hence $g'(1)=0$.

Furthermore,
\[
g''(s)
=\alpha(1-\alpha)\bigl(s^{-\alpha-1}-1\bigr)\ge0,
\qquad s\in(0,1],
\]
because $s\le1$. Therefore $g'$ is nondecreasing on $(0,1]$. Since $g'(1)=0$, it follows that $g'(s)\le0$ for all $s\in(0,1]$. Hence $g$ is nonincreasing on $[0,1]$, and so
\[
g(s)\ge g(1)=0,
\qquad s\in[0,1].
\]
This proves \eqref{eq:normalized-affinity-l2}, and therefore also \eqref{eq:scalar-affinity-l2}.

Now apply \eqref{eq:scalar-affinity-l2} with $x=p_a$ and $y=q_a$ for each $a\in\mathcal X$. Summing over $a\in\mathcal X$ yields
\[
\sum_{a\in\mathcal X}
\Bigl((1-\alpha)p_a+\alpha q_a-p_a^{1-\alpha}q_a^\alpha\Bigr)
\ge
\frac{\alpha(1-\alpha)}{2}
\sum_{a\in\mathcal X}(p_a-q_a)^2.
\]
Since $p$ and $q$ are probability distributions,
\[
\sum_{a\in\mathcal X}\bigl((1-\alpha)p_a+\alpha q_a\bigr)
=
(1-\alpha)\sum_{a}p_a+\alpha\sum_{a}q_a
=
1.
\]
As a result, the left-hand side is equal to
\[
1-\sum_{a\in\mathcal X}p_a^{1-\alpha}q_a^\alpha.
\]
Thus,
\[
1-\sum_{a\in\mathcal X}p_a^{1-\alpha}q_a^\alpha
\ge
\frac{\alpha(1-\alpha)}{2}
\sum_{a\in\mathcal X}(p_a-q_a)^2,
\]
as desired.
\end{proof}

\begin{lemma}[Power threshold-loss bound]
\label{lem:power-threshold-loss}
Let $\mathcal X$ be finite or countable. Let $P$ and $Q$ be probability
measures on $\mathcal X^\infty$. For $t\ge1$, write $P_{<t}$ and $Q_{<t}$
for the marginals on $\mathcal X^{t-1}$, and define
\[
L_t(x_{<t})
:=
\frac{Q_{<t}(x_{<t})}{P_{<t}(x_{<t})}
\]
whenever $P_{<t}(x_{<t})>0$, with an arbitrary value, say $0$, on
$P_{<t}$-null histories. Let $X\sim P$. Then, for every
$\alpha\in(0,1)$ and every $c>0$,
\[
\mathbb E_{X\sim P}
\sum_{t\ge1}
\mathbf 1\{L_t\ge c\}
\sum_{a\in\mathcal X}
\bigl(P(a\mid X_{<t})-Q(a\mid X_{<t})\bigr)^2
\le
\frac{2}{\alpha(1-\alpha)}c^{-\alpha}.
\]
\end{lemma}

\begin{proof}
Fix $\alpha\in(0,1)$ and set $R_t:=L_t^\alpha$. On histories with $P_{<t}(X_{<t})=0$, define the conditional distributions
arbitrarily; these histories have no effect under $P$.

For a fixed history $X_{<t}$, write
\[
p_a:=P(a\mid X_{<t}),
\qquad
q_a:=Q(a\mid X_{<t}).
\]
Define
\[
d_t^\alpha(P,Q)
:=
1-\sum_{a\in\mathcal X}p_a^{1-\alpha}q_a^\alpha.
\]
We use the convention that $0^{1-\alpha}q^\alpha=0$ for $q\ge0$.

On the event $p_{X_t}>0$, the likelihood ratio updates as
$L_{t+1}=L_t\frac{q_{X_t}}{p_{X_t}}$. Therefore,
\[
R_{t+1}
=R_t\left(\frac{q_{X_t}}{p_{X_t}}\right)^\alpha.
\]
Taking conditional expectation under $P$ given $X_{<t}$ gives
\[
\begin{aligned}
\mathbb E_P[R_{t+1}\mid X_{<t}]
&=
R_t
\sum_{a:p_a>0}
p_a\left(\frac{q_a}{p_a}\right)^\alpha  \\
&=
R_t
\sum_{a:p_a>0}
p_a^{1-\alpha}q_a^\alpha  \\
&=
R_t
\sum_{a\in\mathcal X}
p_a^{1-\alpha}q_a^\alpha  \\
&=
R_t\bigl(1-d_t^\alpha(P,Q)\bigr).
\end{aligned}
\]
Hence
\[
\mathbb E_P[R_t-R_{t+1}\mid X_{<t}]
=
R_t d_t^\alpha(P,Q).
\]
Taking expectations yields
\[
\mathbb E_P[R_t d_t^\alpha(P,Q)]
=
\mathbb E_P[R_t-R_{t+1}].
\]

Now sum this identity from $t=1$ to $T$. Since $R_t\ge0$,
\[
\begin{aligned}
\sum_{t=1}^T
\mathbb E_P[R_t d_t^\alpha(P,Q)]
&=
\sum_{t=1}^T
\mathbb E_P[R_t-R_{t+1}] \\
&=
\mathbb E_P[R_1]-\mathbb E_P[R_{T+1}] \\
&\le
\mathbb E_P[R_1].
\end{aligned}
\]
Because $X_{<1}$ is the empty history, $L_1=1$ and therefore $R_1=1$. Thus,
\[
\sum_{t=1}^T
\mathbb E_P[R_t d_t^\alpha(P,Q)]
\le 1.
\]
Letting $T\to\infty$ and using monotone convergence gives
\[
\sum_{t\ge1}
\mathbb E_P[R_t d_t^\alpha(P,Q)]
\le 1.
\]

Next, on the event $\{L_t\ge c\}$,
\[
\mathbf 1\{L_t\ge c\}
\le
c^{-\alpha}L_t^\alpha
=
c^{-\alpha}R_t.
\]
By Lemma~\ref{lem:affinity-to-l2}, for every history $X_{<t}$,
\[
\sum_{a\in\mathcal X}(p_a-q_a)^2
\le
\frac{2}{\alpha(1-\alpha)}
d_t^\alpha(P,Q).
\]
Therefore,
\[
\begin{aligned}
&\mathbb E_P
\sum_{t\ge1}
\mathbf 1\{L_t\ge c\}
\sum_{a\in\mathcal X}
\bigl(P(a\mid X_{<t})-Q(a\mid X_{<t})\bigr)^2 \\
&\le
\frac{2}{\alpha(1-\alpha)}
\mathbb E_P
\sum_{t\ge1}
\mathbf 1\{L_t\ge c\}
d_t^\alpha(P,Q) \\
&\le
\frac{2}{\alpha(1-\alpha)}c^{-\alpha}
\sum_{t\ge1}
\mathbb E_P[R_t d_t^\alpha(P,Q)] \\
&\le
\frac{2}{\alpha(1-\alpha)}c^{-\alpha}.
\end{aligned}
\]
This proves the lemma.
\end{proof}

\begin{lemma}[Thresholded likelihood-ratio loss]
\label{lem:threshold-lr-loss}
Let \(P\) and \(Q\) be probability measures on \(\mathcal X^\infty\), and define
\[
L_t:=\frac{Q(X_{<t})}{P(X_{<t})},
\]
with $X \sim P$  and \(L_t=0\) on \(P\)-null histories. There exists a universal constant
\(K_{\rm PH}<\infty\) such that, for every \(c>0\),
\[
\mathbb E_{X \sim P}
\sum_{t\ge1}
\mathbf 1\{L_t\ge c\}
\sum_{a\in\mathcal X}
\bigl(P(a\mid X_{<t})-Q(a\mid X_{<t})\bigr)^2
\le
\frac{2K_{\rm PH}}{c}.
\]
Using the static-MDL quadratic-loss bound of
\cite[Corollary 14]{poland2005asymptotics}, one may take \(K_{\rm PH}=21\).
\end{lemma}

\begin{proof}
Let $\tau:=\inf\{t\ge1:L_t\ge c\}$ be the first threshold-crossing time. If \(\tau=\infty\), the thresholded loss is
zero.

Condition on \(\{\tau<\infty\}\) and on \(X_{<\tau}=x\). Let \(P^x\) and \(Q^x\) be the conditional continuation measures after history \(x\). For a future string \(y\),
\[
\frac{Q^x(y)}{P^x(y)}
=
\frac{Q(xy)/Q(x)}{P(xy)/P(x)}
=
\frac{L_{\tau+\ell(y)}}{L_\tau}.
\]
Therefore the rule that uses \(Q\) exactly when \(L_{\tau+\ell(y)}\ge c\) is equivalent to exact static MDL on the two-model class \(\{P^x,Q^x\}\) with unnormalized weights
\[
\bar w_P=1,
\qquad
\bar w_Q=\frac{L_\tau}{c}.
\]
Indeed, this static MDL rule selects \(Q^x\) exactly when
\[
\bar w_Q Q^x(y)\ge \bar w_P P^x(y),
\]
which is equivalent to \(L_{\tau+\ell(y)}\ge c\).

By the exact static-MDL quadratic-loss bound of \cite[Corollary 14]{poland2005asymptotics}, the conditional expected cumulative squared loss of this two-model static MDL predictor is at most
\[
K_{\rm PH}\left(1+\frac{L_\tau}{c}\right),
\]
because the normalized weight of the true model \(P^x\) is
$\frac{1}{1+L_\tau/c}$.

It remains to average over the stopping event. With the convention \(L_t=0\) on \(P\)-null histories, \((L_t)_{t\ge1}\) is a nonnegative \(P\)-supermartingale with \(L_1=1\). Optional stopping applied to \(L_{\tau\wedge n}\), followed by Fatou's lemma, gives
\[
\mathbb E_P[L_\tau \cdot \bm{1}_{\{\tau<\infty\}}] \leq \mathbb E_P[\liminf_{n \to \infty}L_{\tau \land n} \cdot \bm{1}_{\{\tau<\infty\}}] \leq \liminf_{n \to \infty}  \mathbb E_P[L_{\tau \land n} \cdot \bm{1}_{\{\tau<\infty\}}]\le1.
\]
Moreover, since \(L_\tau\ge c\) on \(\{\tau<\infty\}\),
$\mathbb P_P(\tau<\infty)\le \frac1c$. Hence
\begin{align*}
    \mathbb E_P
\sum_{t\ge1}
\mathbf 1\{L_t\ge c\}
\sum_{a\in\mathcal X}
\bigl(P(a\mid X_{<t})-Q(a\mid X_{<t})\bigr)^2
&\le
K_{\rm PH}
\left(
\mathbb P_P(\tau<\infty)
+
\frac{\mathbb E_P[L_\tau\cdot \bm{1}_{\{\tau<\infty\}}]}{c}
\right) \\
&\le \frac{2K_{\rm PH}}{c}.
\end{align*}
\end{proof}
\begin{proof}[Proof of Theorem~\ref{thm:intro-positive}]
For each \(\nu\in\mathcal M\), define
\[
L_t^\nu:=\frac{\nu(X_{<t})}{\mu(X_{<t})},
\qquad
c_\nu:=2^{-C}\left(\frac{w_\mu}{w_\nu}\right)^\lambda .
\]
By \eqref{eq:additive-approx-maxform}, on the event \(\{\nu_t=\nu\}\),
\[
w_\nu^\lambda\nu(X_{<t})
\ge
2^{-C}w_\mu^\lambda\mu(X_{<t}),
\]
and hence \(L_t^\nu\ge c_\nu\). Therefore
\begin{equation}\label{eq:threshold-template}
\begin{aligned}
S_\infty(\mu,\rho)
&\le
\sum_{\nu\in\mathcal M}
\mathbb E_\mu\sum_{t\ge1}
\mathbf 1\{L_t^\nu\ge c_\nu\}
\sum_{a\in\mathcal X}
\bigl(\mu(a\mid X_{<t})-\nu(a\mid X_{<t})\bigr)^2 .
\end{aligned}
\end{equation}
This is the common threshold template for all \(\lambda\ge1\).

We combine the two pairwise threshold estimates directly. Fix \(0<p<1\) and set
\[
\alpha:=\frac{p}{\lambda},
\qquad
r_\nu:=\frac{w_\nu}{w_\mu}.
\]
By Lemmas~\ref{lem:power-threshold-loss} and~\ref{lem:threshold-lr-loss},
\[
\mathbb E_\mu\sum_{t\ge1}
\mathbf 1\{L_t^\nu\ge c_\nu\}
\sum_{a\in\mathcal X}
\bigl(\mu(a\mid X_{<t})-\nu(a\mid X_{<t})\bigr)^2
\le
\min\{A_pr_\nu^p,Br_\nu^\lambda\},
\]
where
\[
A_p:=
\frac{2\lambda^2}{p(\lambda-p)}2^{pC/\lambda},
\qquad
B:=2K_{\rm PH}2^C.
\]
For \(0<p<1\le\lambda\), the elementary interpolation inequality
\[
\min\{A_pr^p,Br^\lambda\}
\le
A_p^{\frac{\lambda-1}{\lambda-p}}
B^{\frac{1-p}{\lambda-p}}r
\qquad(r>0)
\]
holds. Indeed, the left-hand side divided by \(r\) is bounded by the value at the crossing point of \(A_pr^p\) and \(Br^\lambda\). Hence
\[
S_\infty(\mu,\rho)
\le
2^{C/\lambda}
\left(
\frac{2\lambda^2}{p(\lambda-p)}
\right)^{\frac{\lambda-1}{\lambda-p}}
(2K_{\rm PH})^{\frac{1-p}{\lambda-p}}
w_\mu^{-1}.
\]
In particular, taking \(p=1/2\) gives
\[
S_\infty(\mu,\rho)
\le
2^{C/\lambda}
\left(
\frac{4\lambda^2}{\lambda-\frac12}
\right)^{\frac{\lambda-1}{\lambda-\frac12}}
(2K_{\rm PH})^{\frac{1}{2\lambda-1}}
w_\mu^{-1}.
\]
Let
\[
C_\lambda:=
\left(
\frac{4\lambda^2}{\lambda-\frac12}
\right)^{\frac{\lambda-1}{\lambda-\frac12}}
(2K_{\rm PH})^{\frac{1}{2\lambda-1}}.
\]
For \(\lambda\ge1\), we have
\[
0\le \frac{\lambda-1}{\lambda-\frac12}\le1,
\qquad
\frac{4\lambda^2}{\lambda-\frac12}\le 8\lambda,
\qquad
\frac1{2\lambda-1}\le1.
\]
Thus \(C_\lambda\le16K_{\rm PH}\lambda\), and therefore
\[
S_\infty(\mu,\rho)
\le
C_\lambda 2^{C/\lambda}w_\mu^{-1}
\le
16K_{\rm PH}\lambda\,2^{C/\lambda}w_\mu^{-1}.
\]
This proves the theorem.
\end{proof}

\begin{proof}[Proof of Proposition\ref{prop:additive-positive-sharpness}]
Let
\[
A:=2^{C/\lambda}(w^{-1}-1).
\]
We first prove the following floor lower bound: for every integer \(k\ge1\), there is a finite binary model class, weights, a true environment \(\mu\) with \(w_\mu=w\), and a valid additive-\(C\) Approx-MDL predictor
\(\rho\) such that
\begin{equation}
\label{eq:sharpness-floor-bound}
S_\infty(\mu,\rho)
\ge
\frac{k}{2}
\left\lfloor
2^{(C+1-k)/\lambda}(w^{-1}-1)
\right\rfloor .
\end{equation}

Let the true environment be the deterministic all-zero measure on the binary alphabet,
\[
\mu(000\cdots)=1,
\qquad
w_\mu=w.
\]
Fix \(k\ge1\). Set
\[
r:=2^{(k-1-C)/\lambda},
\qquad
m:=
\left\lfloor
\frac{1-w}{rw}
\right\rfloor
=
\left\lfloor
2^{(C+1-k)/\lambda}(w^{-1}-1)
\right\rfloor .
\]
Add \(m\) false models \(\nu_1,\ldots,\nu_m\), each with weight
\[
w_{\nu_i}:=rw.
\]
Then
\[
w_\mu+\sum_{i=1}^m w_{\nu_i}
=
w+mrw
\le
w+(1-w)
=
1.
\]
If \(m=0\), then the floor lower bound \eqref{eq:sharpness-floor-bound} is trivial. Hence assume \(m\ge1\).

The false models make their errors in round-robin order. For
\[
t_{i,j}:=(j-1)m+i,
\qquad
1\le i\le m,\quad 1\le j\le k,
\]
define \(\nu_i\) along the true path by
\[
\nu_i(0\mid 0^{t_{i,j}-1})=\frac12,
\qquad
\nu_i(1\mid 0^{t_{i,j}-1})=\frac12
\]
at its designated error times, and by
\[
\nu_i(0\mid 0^{t-1})=1
\]
at all other times along the path \(0^\infty\). Away from the path \(0^\infty\), define \(\nu_i\) arbitrarily, say deterministically toward \(0\). This defines a probability measure on \(\{0,1\}^\infty\).

We now define the selector. At history \(0^{t_{i,j}-1}\), select \(\nu_i\). At all other histories, select an exact minimizer of the MDL objective over the finite class. We verify that the prescribed selections are valid additive-\(C\) Approx-MDL choices.

Before time \(t_{i,j}\), model \(\nu_i\) has made exactly \(j-1\) fair-coin mistakes along the true path. Therefore
\[
\frac{\nu_i(0^{t_{i,j}-1})}{\mu(0^{t_{i,j}-1})}
=
2^{-(j-1)}.
\]
Its weighted likelihood relative to the true model is
\[
A_j
:=
\frac{w_{\nu_i}^{\lambda}\nu_i(0^{t_{i,j}-1})}
     {w_\mu^\lambda\mu(0^{t_{i,j}-1})}
=
r^\lambda 2^{-(j-1)}
=
2^{k-j-C}.
\]
Since \(j\le k\), we have
\[
A_j\ge 2^{-C}.
\]

We also need admissibility relative to the global maximum, not merely relative to the true model. At the \(i\)-th step of round \(j\), the models \(\nu_1,\ldots,\nu_{i-1}\) have already been used in the current round and have therefore made \(j\) mistakes, while \(\nu_i,\ldots,\nu_m\) have made
only \(j-1\) mistakes. Thus, the selected model \(\nu_i\) attains the largest weighted likelihood among the false models. Hence, the global maximum, measured relative to the true model, is at most \(\max\{1, A_j\}\). Since
\[
A_j\ge 2^{-C}\max\{1,A_j\},
\]
the choice of \(\nu_i\) satisfies
\[
w_{\nu_i}^{\lambda}\nu_i(0^{t_{i,j}-1})
\ge
2^{-C}
\max_{\eta\in\mathcal M}
w_\eta^\lambda\eta(0^{t_{i,j}-1}).
\]
Therefore, the selector is a valid additive-\(C\) Approx-MDL selector.

Each designated selection incurs a squared loss
\[
\left(1-\frac12\right)^2+\left(0-\frac12\right)^2
=
\frac12.
\]
Since \(\mu\) is deterministic, the expectation is concentrated on the path \(0^\infty\). There are \(km\) designated error times, so
\[
S_\infty(\mu,\rho)
\ge
\frac{km}{2}
=
\frac{k}{2}
\left\lfloor
2^{(C+1-k)/\lambda}(w^{-1}-1)
\right\rfloor .
\]
This proves \eqref{eq:sharpness-floor-bound}.

It remains to derive the stated simplified bound. Recall that
\[
A=2^{C/\lambda}(w^{-1}-1).
\]
If \(A\le1\), then
\[
A-1\le0,
\]
and the claimed lower bound
\[
S_\infty(\mu,\rho)\ge c\lambda(A-1)
\]
is trivial, since \(S_\infty(\mu,\rho)\ge0\). If the statement is written with the positive part \((A-1)_+\), the same case gives the trivial bound \(S_\infty(\mu,\rho)\ge0\).

Now suppose \(A>1\). We split into two subcases. First assume \(1<A\le 4e\). Choose $k:=\left\lfloor 1+\lambda\log_2 A\right\rfloor$. Then $k\ge1$, and since \(k\le 1+\lambda\log_2 A\),
\[
2^{(C+1-k)/\lambda}(w^{-1}-1)
=
A\,2^{(1-k)/\lambda}
\ge
1.
\]
Therefore the floor in \eqref{eq:sharpness-floor-bound} is at least \(1\), and hence
\[
S_\infty(\mu,\rho)
\ge
\frac{k}{2}
\ge
\frac{\lambda}{2}\log_2 A.
\]
Since the function $A\mapsto \frac{\log_2 A}{A-1}$ is positive and continuous on the compact interval \((1,4e]\) after taking
its limit at \(A=1\), there exists an absolute constant \(\gamma>0\) such
that
\[
\log_2 A\ge \gamma(A-1)
\qquad
\text{for all }1<A\le4e.
\]
Thus
\[
S_\infty(\mu,\rho)
\ge
\frac{\gamma}{2}\lambda(A-1).
\]

Next assume \(A>4e\). Choose $k:=\left\lfloor \frac{\lambda}{\ln 2}\right\rfloor$. Since \(\lambda\ge1\), this gives \(k\ge1\) and
\[
k\ge \frac{\lambda}{2\ln 2}.
\]
Moreover, \(k\le \lambda/\ln 2\), and hence
\[
2^{-k/\lambda}\ge 2^{-1/\ln 2}=e^{-1}.
\]
Therefore
\[
2^{(C+1-k)/\lambda}(w^{-1}-1)
=
A\,2^{(1-k)/\lambda}
\ge
A e^{-1}
>
4.
\]
Thus, the floor is at least half of the quantity inside it:
\[
\left\lfloor
2^{(C+1-k)/\lambda}(w^{-1}-1)
\right\rfloor
\ge
\frac12 A\,2^{(1-k)/\lambda}
\ge
\frac{A}{2e}.
\]
Using \eqref{eq:sharpness-floor-bound}, we obtain
\[
S_\infty(\mu,\rho)
\ge
\frac{k}{2}\cdot \frac{A}{2e}
\ge
\frac{\lambda}{2\ln 2}\cdot \frac{A}{4e}
=
\frac{\lambda A}{8e\ln 2}.
\]
Since \(A>1\), this also implies
\[
S_\infty(\mu,\rho)
\ge
\frac{\lambda(A-1)}{8e\ln 2}.
\]

Combining the two subcases, there exists an absolute constant \(c>0\) such that, for all \(A>1\),
\[
S_\infty(\mu,\rho)
\ge
c\lambda(A-1).
\]
Together with the trivial case \(A\le1\), this proves
\[
S_\infty(\mu,\rho)
\ge
c\lambda
\Bigl(2^{C/\lambda}(w_\mu^{-1}-1)-1\Bigr).
\]
If the proposition is stated with a positive part, the same proof gives
\[
S_\infty(\mu,\rho)
\ge
c\lambda
\Bigl(2^{C/\lambda}(w_\mu^{-1}-1)-1\Bigr)_+.
\]
This proves the proposition.
\end{proof}

\begin{proof}[Proof of Corollary~\ref{cor:intro-universal}]\label{app:proof-cor-universal}
For the universal class with \(w_\nu=2^{-K(\nu)}\), Kraft's inequality for the underlying prefix machine gives
\(\sum_{\nu\in\mathcal M} w_\nu\le 1\). Since the true estimable environment
\(\mu\) belongs to \(\mathcal M\), Theorem~\ref{thm:intro-positive} yields
\[
S_\infty(\mu,\rho)=\sum_{t\ge1} e_t(\mu,\rho)<\infty .
\]
As each \(e_t(\mu,\rho)\ge0\), the terms of this convergent nonnegative series
satisfy \(e_t(\mu,\rho)\to0\).
\end{proof}

\begin{lemma}[Deterministic refinement of an estimable measure]
\label{lem:deterministic-refinement}
There exists a constant \(c_{\rm det}<\infty\) such that for every estimable
measure \(\nu\) and every finite binary string \(x\) with \(\nu(x)>0\), there is
a deterministic estimable measure \(\delta_{\omega}\) satisfying
\(\delta_{\omega}(x)=1\) and
\[
K(\delta_{\omega})
\le
K(\nu)
+
\log_2\nu(x)^{-1}
+
c_{\rm det}\log_2\!\bigl(\log_2\nu(x)^{-1}+2\bigr)
+
c_{\rm det}.
\]
\end{lemma}

\begin{proof}
By the equivalence between generative and probability-computing descriptions for
estimable measures, we may use a probability-computing description of \(\nu\)
of length \(K(\nu)+O(1)\). This \(O(1)\) overhead is absorbed into
\(c_{\rm det}\).

Let \(B=-\log_2\nu(x)\). From the probability-computing description of \(\nu\),
construct the standard arithmetic intervals
\((I_y)_{y\in\{0,1\}^*}\), whose endpoints are uniformly computable and whose
lengths satisfy \(|I_y|=\nu(y)\). Since \(|I_x|=\nu(x)=2^{-B}>0\), a standard
dyadic-interval covering argument gives a binary string \(r\) such that the
dyadic interval \(J=J_r\) has closure contained in the interior of \(I_x\) and
\[
|r|\le \lceil B\rceil+c_0
\]
for a universal constant \(c_0\).

We hard-wire \(r\) into the decoder. The decoder does not need to know \(x\).
It maintains a rational interval \(J_y\) after having output the prefix \(y\),
with \(J_y\subseteq J\cap I_y\) and \(|J_y|>0\). Initially,
\(y=\epsilon\) and \(J_\epsilon=J\).

Given \(J_y\subseteq I_y\), the decoder dovetails over
\(a\in\{0,1\}\) and rational open intervals \(J'\). For each candidate \(J'\),
it uses computable approximations to the endpoints of \(I_{ya}\) to certify the
strict containment
\[
\overline{J'}\subseteq J_y\cap \operatorname{int}(I_{ya}).
\]
This condition is semidecidable because the containment is strict. Moreover,
such a pair \((a,J')\) exists: the interval \(J_y\) has positive length and the
children \(I_{y0}\) and \(I_{y1}\) partition \(I_y\) up to endpoints, so one of
the intersections \(J_y\cap I_{ya}\) has positive interior. Hence the dovetailing
search halts. The decoder outputs this bit \(a\) and sets \(J_{ya}=J'\).

Iterating this procedure produces a computable infinite sequence \(\omega\).
Because \(J\) is contained in the interior of \(I_x\), during the first
\(|x|\) steps only the branch consistent with \(x\) can have positive
intersection with the maintained interval. Hence the first \(|x|\) bits of
\(\omega\) are exactly \(x\), and therefore \(\delta_\omega(x)=1\).

Since \(\omega\) is computable, the Dirac measure \(\delta_\omega\) is estimable:
on input \(y\), compute the first \(|y|\) bits of \(\omega\) and check whether
\(y\) is a prefix.

The generator for \(\delta_\omega\) consists of the probability-computing
description of \(\nu\), the self-delimiting encoding of \(r\), and a fixed
decoder. The string \(r\) costs
\[
|r|+O(\log(|r|+2))
\]
bits to encode self-delimitingly. Since \(|r|\le \lceil B\rceil+c_0\), the
description length is
\[
K(\nu)+B+O(\log(B+2)).
\]
Absorbing all universal constants into \(c_{\rm det}\) proves the claim.
\end{proof}

\begin{proof}[Proof of Lemma~\ref{lem:pointwise-overfitting}]
Let \(c_{\rm code}\) be the constant in
Lemma~\ref{lem:coding-relative-to-measure}.

First we prove the lower bound on the optimum. For any \(\nu\in\mathcal M\),
Lemma~\ref{lem:coding-relative-to-measure} gives
\[
K_m(x)
\le
K(\nu)+\log_2\nu(x)^{-1}+c_{\rm code}.
\]
Multiplying by \(\lambda\in(0,1)\) and using
\[
\lambda\log_2\nu(x)^{-1}
\le
\log_2\nu(x)^{-1},
\]
we get
\[
\lambda K_m(x)
\le
\lambda K(\nu)+\log_2\nu(x)^{-1}+\lambda c_{\rm code}
=
L_\lambda(\nu;x)+\lambda c_{\rm code}.
\]
Taking the minimum over \(\nu\) yields
\[
\operatorname{OPT}_\lambda(x)
\ge
\lambda K_m(x)-\lambda c_{\rm code}.
\]
Thus the second claim holds with \(c_\lambda=\lambda c_{\rm code}\).

We now prove the deterministic approximate-minimizer claim. Let
\(\nu^\star\) be an exact minimizer of \(L_\lambda(\cdot;x)\), and write
\[
B:=
\log_2\nu^\star(x)^{-1}.
\]
Since the universal class contains deterministic extensions of \(x\), the
optimum is finite, and hence \(B<\infty\). By
Lemma~\ref{lem:deterministic-refinement}, there exists a deterministic
estimable measure \(\nu_x\) with \(\nu_x(x)=1\) and
\[
K(\nu_x)
\le
K(\nu^\star)
+
B
+
c_{\rm det}\log_2(B+2)
+
c_{\rm det}.
\]
Therefore,
\[
L_\lambda(\nu_x;x)
=
\lambda K(\nu_x)
\le
\lambda K(\nu^\star)
+
\lambda B
+
\lambda c_{\rm det}\log_2(B+2)
+
\lambda c_{\rm det}.
\]
Since
\[
\operatorname{OPT}_\lambda(x)
=
L_\lambda(\nu^\star;x)
=
\lambda K(\nu^\star)+B,
\]
we obtain
\[
L_\lambda(\nu_x;x)-\operatorname{OPT}_\lambda(x)
\le
\lambda c_{\rm det}\log_2(B+2)
+
\lambda c_{\rm det}
-
(1-\lambda)B.
\]
Because \(0<\lambda<1\), the right-hand side is bounded above uniformly over
\(B\ge0\). Define
\[
C_\lambda
:=
\sup_{B\ge0}
\left\{
\lambda c_{\rm det}\log_2(B+2)
+
\lambda c_{\rm det}
-
(1-\lambda)B
\right\}
<\infty.
\]
Then
\[
L_\lambda(\nu_x;x)
\le
\operatorname{OPT}_\lambda(x)+C_\lambda,
\]
as required.
\end{proof}

\begin{proof}[Proof of Theorem~\ref{thm:intro-negative-additive}]
Let the true environment be the fair-coin measure
\[
\mu:=\mathrm{Bern}(1/2)^\infty.
\]
This measure is estimable, so \(\mu\in\mathcal M\). For each history
\(x\in\{0,1\}^*\), let \(S(x):=\nu_x\) be the deterministic measure constructed
in Lemma~\ref{lem:pointwise-overfitting}. The lemma implies that there exists a
constant \(C<\infty\), depending only on \(\lambda\) and the reference machines,
such that for every history \(x\),
\[
L_\lambda(S(x);x)
\le
\operatorname{OPT}_\lambda(x)+C.
\]
Hence \(S\) is a valid additive-\(C\) Approx-MDL selector. Define the predictor by
\[
\rho_t(\cdot\mid x_{<t}) := S(x_{<t})(\cdot\mid x_{<t}).
\]

Now fix any history $x$. Since $S(x)$ is deterministic, there exists a bit
$b(x)\in\{0,1\}$ such that
\[
S(x)(\cdot\mid x)=\delta_{b(x)}.
\]
On the other hand, the true environment $\mu$ always predicts
\[
\mu(\cdot\mid x)=(1/2,1/2).
\]
Therefore the one-step squared loss is exactly
\[
\sum_{a\in\{0,1\}}
\Bigl(\mu(a\mid x)-S(x)(a\mid x)\Bigr)^2
=
\left(\frac12-1\right)^2+\left(\frac12-0\right)^2
=
\frac12.
\]
This holds for every history $x$, so for every $t\ge1$, the expected error remains constant:
\[
\mathbb E_\mu\!\left[
\sum_{a\in\{0,1\}}
\Bigl(\mu(a\mid X_{<t})-\rho_t(a\mid X_{<t})\Bigr)^2
\right]
=
\frac12.
\]
Hence, the cumulative squared error diverges:
\[
S_\infty(\mu,\rho)
=
\sum_{t\ge1}\frac12
=
+\infty.
\]
Because the expected single-step prediction error fails to converge to zero, the predictor $\rho$ does not universally learn the estimable environment $\mu$.
\end{proof}

\section{Full proofs for multiplicative Approx-MDL}
\label{app:multiplicative-proofs}

\begin{proof}[Proof of Lemma~\ref{lem:weighted-mdl-lower-tail}]
Define the mixture over all estimable measures weighted by their prior:
\[
\xi_{\lambda}(x):=\sum_{\nu\in\mathcal M}2^{-\lambda K(\nu)}\nu(x).
\]
For every $n$,
\[
\sum_{x\in\mathcal X^n}\xi_{\lambda}(x)
=\sum_{\nu\in\mathcal M}2^{-\lambda K(\nu)}\sum_{x\in\mathcal X^n}\nu(x)
\le
\sum_{\nu\in\mathcal M}2^{-\lambda K(\nu)}
\le
\sum_{\nu\in\mathcal M}2^{-K(\nu)}
\le 1,
\]
where we used $\lambda\ge1$ and Kraft's inequality for the prefix-free codes underlying the Kolmogorov complexity. Thus $\xi_{\lambda}$ is a semimeasure.

For any $\nu\in\mathcal M$ and any string $x$,
\[
\xi_{\lambda}(x)\ge 2^{-\lambda K(\nu)}\nu(x),
\]
so
\[
-\log_2\xi_{\lambda}(x)
\le
\lambda K(\nu)+\log_2\nu(x)^{-1}
=
L_{\lambda}(\nu;x).
\]
Taking the minimum over $\nu$ gives
\begin{equation}\label{eq:xi-below-opt}
-\log_2\xi_{\lambda}(x)
\le
\operatorname{OPT}_{\lambda}(x).
\end{equation}

Now define the nonnegative random variable
\[
Z_n:=\frac{\xi_{\lambda}(X_{1:n})}{\mu(X_{1:n})},
\]
with the convention $Z_n:=0$ whenever $\mu(X_{1:n})=0$. Then
\[
\mathbb E_\mu[Z_n]
=
\sum_{x\in\mathcal X^n}\mu(x)\frac{\xi_{\lambda}(x)}{\mu(x)}
=
\sum_{x\in\mathcal X^n}\xi_{\lambda}(x)
\le 1.
\]
By Markov's inequality,
\[
\Pr_\mu[Z_n\ge 2^c]\le 2^{-c}.
\]
Equivalently, with probability at least $1-2^{-c}$,
\[
\xi_{\lambda}(X_{1:n})\le 2^c\mu(X_{1:n}),
\]
that is,
\[
-\log_2\mu(X_{1:n})
\le
-\log_2\xi_{\lambda}(X_{1:n})+c.
\]
Combining this with \eqref{eq:xi-below-opt} yields
\[
\operatorname{OPT}_{\lambda}(X_{1:n})
\ge
\log_2\mu(X_{1:n})^{-1}-c
\]
with probability at least $1-2^{-c}$.
\end{proof}

\begin{proof}[Proof of Theorem~\ref{thm:intro-negative-multiplicative}]
We work on the binary alphabet $\mathcal X=\{0,1\}$, and let
\[
\mu:=\operatorname{Bern}(1/2)^\infty
\]
be the fair-coin product measure. This environment is estimable, hence
$\mu\in\mathcal M$. We split the proof into two cases.

\paragraph{Case 1: $0<\lambda<1$.}
Here, we use the pointwise overfitting construction from the additive lower bound. By Lemma~\ref{lem:pointwise-overfitting}, there exists a constant $C_0<\infty$ such that for every finite string $x\in\{0,1\}^*$, there is a
deterministic estimable measure $\nu_x\in\mathcal M$ satisfying
\[
L_\lambda(\nu_x;x)\le \operatorname{OPT}_\lambda(x)+C_0.
\]
Moreover, the same lemma gives a constant $C_1<\infty$ such that
\[
\operatorname{OPT}_\lambda(x)\ge \lambda K_m(x)-C_1
\qquad\text{for all }x\in\{0,1\}^*.
\]

For $X_{1:n}\sim\mu$, the monotone complexity $K_m(X_{1:n})$ tends to infinity
in probability. Indeed, for every fixed $M<\infty$, the number of strings
$x\in\{0,1\}^n$ with $K_m(x)\le M$ is at most $2^{M+1}$, and each such string
has $\mu$-probability $2^{-n}$. Hence
\[
\mathbb P_\mu[K_m(X_{1:n})\le M]\le 2^{M+1-n}\to0.
\]
Combining this with the lower bound
\[
\operatorname{OPT}_\lambda(x)\ge \lambda K_m(x)-C_1
\]
shows that
\[
\operatorname{OPT}_\lambda(X_{1:n})\to\infty
\qquad\text{in }\mu\text{-probability}.
\]

Define the event
\[
A_n:=
\left\{
\operatorname{OPT}_\lambda(X_{1:n})\ge \frac{C_0}{\epsilon}
\right\}.
\]
Then
\[
\mathbb P_\mu(A_n)\to1.
\]
On $A_n$, the deterministic overfitting model $\nu_{X_{1:n}}$ satisfies
\[
L_\lambda(\nu_{X_{1:n}};X_{1:n})
\le
\operatorname{OPT}_\lambda(X_{1:n})+C_0
\le
(1+\epsilon)\operatorname{OPT}_\lambda(X_{1:n}).
\]
Thus, with probability tending to one, $\nu_{X_{1:n}}$ is a valid
multiplicative-$(1+\epsilon)$ Approx-MDL choice.

Now define a selector $S$ by
\[
S(x):=
\begin{cases}
\nu_x, & \text{if }
L_\lambda(\nu_x;x)\le (1+\epsilon)\operatorname{OPT}_\lambda(x),\\
\nu_x^\star, & \text{otherwise},
\end{cases}
\]
where $\nu_x^\star$ is any exact minimizer of $L_\lambda(\cdot;x)$. Such an
exact minimizer exists for $\lambda>0$. Define
\[
\rho_{n+1}(\cdot\mid x_{1:n}) := S(x_{1:n})(\cdot\mid x_{1:n}).
\]
By construction, $\rho$ is a valid multiplicative-$(1+\epsilon)$ Approx-MDL
predictor.

On the event $A_n$, the selected model is the deterministic measure
$\nu_{X_{1:n}}$. Therefore there exists a bit $b(X_{1:n})\in\{0,1\}$ such that
\[
\rho_{n+1}(\cdot\mid X_{1:n})=\delta_{b(X_{1:n})}.
\]
Since the true conditional distribution is
\[
\mu(\cdot\mid X_{1:n})=(1/2,1/2),
\]
The squared prediction error on $A_n$ equals
\[
\sum_{a\in\{0,1\}}
\bigl(\mu(a\mid X_{1:n})-\rho_{n+1}(a\mid X_{1:n})\bigr)^2
=
\frac12.
\]
Hence
\[
\mathbb E_\mu
\left[
\sum_{a\in\{0,1\}}
\bigl(\mu(a\mid X_{1:n})-\rho_{n+1}(a\mid X_{1:n})\bigr)^2
\right]
\ge
\frac12\,\mathbb P_\mu(A_n).
\]
Since $\mathbb P_\mu(A_n)\to1$, the expected one-step squared prediction error
is at least $1/4$ for all sufficiently large $n$. Summing over $n$ gives
\[
S_\infty(\mu,\rho)=+\infty.
\]

\paragraph{Case 2: $\lambda\ge1$.}
It remains to handle the strongly regularized regime. Let again
\[
\mu=\operatorname{Bern}(1/2)^\infty.
\]
Choose a rational number $\theta\in(0,1)\setminus\{1/2\}$ sufficiently close to
$1/2$ such that the cross-entropy
\[
h_\theta
:=
\frac12\log_2\theta^{-1}
+
\frac12\log_2(1-\theta)^{-1}
\]
satisfies
\[
h_\theta<1+\frac{\epsilon}{4}.
\]
This is possible by continuity of $h_\theta$ at $\theta=1/2$ and by density of
the rationals. Let
\[
\nu_\theta:=\operatorname{Bern}(\theta)^\infty.
\]
Since $\theta$ is rational, $\nu_\theta$ is estimable, and $K(\nu_\theta)<\infty$.

Choose $N_0$ such that, for all $n\ge N_0$,
\[
\lambda K(\nu_\theta)
\le
\frac{\epsilon}{2}n-4(1+\epsilon).
\]
Under $\mu$, the empirical negative log-likelihood under $\nu_\theta$ converges
in probability to $h_\theta$:
\[
-\frac1n\log_2\nu_\theta(X_{1:n})\to h_\theta
\qquad\text{in }\mu\text{-probability}.
\]
Hence there exists $N_1$ such that, for all $n\ge N_1$,
\[
\mathbb P_\mu
\left[
-\log_2\nu_\theta(X_{1:n})
\le
\left(1+\frac{\epsilon}{2}\right)n
\right]
\ge
\frac9{10}.
\]

For $n\ge N:=\max\{N_0,N_1\}$, define
\[
A_n:=
\left\{
-\log_2\nu_\theta(X_{1:n})
\le
\left(1+\frac{\epsilon}{2}\right)n
\right\}.
\]
On $A_n$,
\[
\begin{aligned}
L_\lambda(\nu_\theta;X_{1:n})
&=
\lambda K(\nu_\theta)+\log_2\nu_\theta(X_{1:n})^{-1} \\
&\le
\frac{\epsilon}{2}n-4(1+\epsilon)
+
\left(1+\frac{\epsilon}{2}\right)n \\
&=
(1+\epsilon)n-4(1+\epsilon) \\
&=
(1+\epsilon)(n-4).
\end{aligned}
\]

Next define
\[
B_n:=\{\operatorname{OPT}_\lambda(X_{1:n})\ge n-4\}.
\]
Since $\mu(X_{1:n})=2^{-n}$ for every binary string of length $n$, Lemma~\ref{lem:weighted-mdl-lower-tail} with $c=4$ gives
\[
\mathbb P_\mu(B_n)\ge 1-2^{-4}=\frac{15}{16}.
\]
Therefore, on $A_n\cap B_n$,
\[
L_\lambda(\nu_\theta;X_{1:n})
\le
(1+\epsilon)(n-4)
\le
(1+\epsilon)\operatorname{OPT}_\lambda(X_{1:n}).
\]
By the union bound,
\[
\mathbb P_\mu(A_n\cap B_n)
\ge
\frac9{10}+\frac{15}{16}-1
=
\frac{67}{80}.
\]

Define the selector $S$ by
\[
S(x):=
\begin{cases}
\nu_\theta, & \text{if }
L_\lambda(\nu_\theta;x)\le(1+\epsilon)\operatorname{OPT}_\lambda(x),\\
\nu_x^\star, & \text{otherwise},
\end{cases}
\]
where $\nu_x^\star$ is any exact minimizer of $L_\lambda(\cdot;x)$. Define
\[
\rho_t(\cdot\mid x_{<t}) := S(x_{<t})(\cdot\mid x_{<t}).
\]
By construction, $\rho$ is a valid multiplicative-$(1+\epsilon)$ Approx-MDL
predictor.

For every $t=n+1\ge N+1$, the preceding bound implies
\[
\mathbb P_\mu(S(X_{<t})=\nu_\theta)\ge \frac{67}{80}.
\]
Whenever $S(X_{<t})=\nu_\theta$, the next-step prediction is the biased
Bernoulli law, whereas the true conditional distribution is
\[
\mu(\cdot\mid X_{<t})=(1/2,1/2).
\]
Thus the squared prediction error on this event is
\[
\sum_{a\in\{0,1\}}
\bigl(\mu(a\mid X_{<t})-\rho_t(a\mid X_{<t})\bigr)^2
=
2\left(\theta-\frac12\right)^2
>
0.
\]
Consequently, for all $t\ge N+1$,
\[
\mathbb E_\mu
\left[
\sum_{a\in\{0,1\}}
\bigl(\mu(a\mid X_{<t})-\rho_t(a\mid X_{<t})\bigr)^2
\right]
\ge
\frac{67}{80}\cdot
2\left(\theta-\frac12\right)^2.
\]
The right-hand side is a strictly positive constant. Summing over $t$ gives
\[
S_\infty(\mu,\rho)=+\infty.
\]
Combining the two cases proves the theorem for every $\lambda>0$ and
$\epsilon>0$.
\end{proof}

\end{document}